\documentclass[twoside,11pt]{article}

\usepackage{jmlr2e}

\usepackage[utf8]{inputenc} % allow utf-8 input
\usepackage[T1]{fontenc}    % use 8-bit T1 fonts
\usepackage{xr} 
\usepackage{hyperref}       % hyperlinks
\usepackage{url}            % simple URL typesetting
\usepackage{booktabs}       % professional-quality tables
\usepackage{amsfonts}       % blackboard math symbols
\usepackage{nicefrac}       % compact symbols for 1/2, etc.
\usepackage{microtype}      % microtypography
\usepackage{amsthm}
\usepackage{comment}
\usepackage[]{algorithm2e}
\usepackage{subcaption}
\usepackage{xcolor}
\usepackage{color}
\usepackage{amsmath}
\usepackage{float}
\newcommand{\CTB}{CTB} % This is the acronym for our CTB algorithm
\newcommand{\CTBone}{CTB$-1$} % This is the acronym for our Bayes CTB algorithm
\newcommand{\CTBtwo}{CTB$-2$} % This is the acronym for our Bayes CTB algorithm
\newcommand{\CTBthree}{CTB$-3$}
\newcommand{\argmax}{\mathop{\mathrm{argmax}}}
\usepackage{mathtools}
\newcommand{\mathbbm}[1]{\text{\usefont{U}{bbm}{m}{n}#1}} % from mathbbm.sty

\theoremstyle{definition}
\newtheorem{definition}{Definition}[section]
\newtheorem{theorem}{Theorem}
\newtheorem{lemma}[theorem]{Lemma}

\firstpageno{1}

\begin{document}

\title{Dueling Bandits with Dependent Arms}

\author{\name Bangrui Chen \email bc496@cornell.edu \\
       \name Peter I. Frazier \email pf98@cornell.edu \\
       \addr School of Operations Research and Information Engineering\\
       296 Rhodes Hall, Cornell University\\
       Ithaca, NY 14853, USA}

%\editor{}

\maketitle

\begin{abstract}
We study dueling bandits with weak utility-based regret when preferences over arms have a total order and carry observable feature vectors. The order is assumed to be determined by these feature vectors, an unknown preference vector, and a known utility function.  This structure introduces dependence between preferences for pairs of arms, and allows learning about the preference over one pair of arms from the preference over another pair of arms. 
We propose an algorithm for this setting called \textit{Comparing The Best} (\CTB), which we show has constant expected cumulative weak utility-based regret. We provide a Bayesian interpretation for \CTB, an implementation appropriate for a small number of arms, and an alternate implementation for many arms that can be used when the input parameters satisfy a decomposability condition. We demonstrate through numerical experiments that \CTB\ with appropriate input parameters outperforms all benchmarks considered.
\end{abstract}

\section{Introduction}
In the dueling bandits problem, we are faced with a collection of arms, and pull a pairs of arms while observing noisy binary feedback indicating which arm is better for each pulled pair.  As in the classical multi-armed bandit problem, we wish to pull arms to quickly learn which arm is best and minimize the number of pulls to suboptimal arms.

Dueling bandits were introduced by \cite{yue2009interactively}, motivated by interactive optimization of web search and other information retrieval systems. The advantage of the dueling bandits formulation over the classical multi-armed bandits formulation in this application setting is that pairwise comparison results can be reliably inferred from implicit feedback, for example through interleaved rankings in \cite{radlinski2008does}, in contrast with cardinal evaluation obtained from explicit feedback, which is typically difficult to obtain, biased, and requires careful calibration \citep{joachims2007evaluating, yue2012k}.

Dueling bandits have been studied most frequently assuming strong regret, in which the regret is 0 if and only if both pulled arms are optimal. Several algorithms have been devised that assume the existence of a Condorcet winner, i.e., one that is preferred in comparison with each other arm.  Algorithms with order-optimal strong regret, $O(N\log(T))$, in this setting include BTM \citep{yue2011beat}, RUCB \citep{zoghi2014relative} and RMED \citep{komiyama2015regret}. \cite{zoghi2015copeland} points out points out that a Condorcet winner does not necessarily exist, and that its probability of existence decreases dramatically with the number of arms. That work instead studies the dueling bandits assuming a Copeland winner, which is guaranteed to exist, and propose two algorithms, CCB and SCB, which achieve $O(N\log(T))$ strong regret in this more general setting.  

The above papers on strong regret bound the binary strong regret, in which the regret is $1$ whenever it is strictly positive. \cite{ailon2014reducing} considered strong utility-based regret, in which each arm has a utility score from which preferences are derived, and the regret for failing to pull the maximum utility arm twice is a function of that maximal arm's utility and the utilities of the pulled arms. 

Bandits have also been considered, though less frequently, in the weak regret setting, introduced by \cite{yue2012k}, in which regret is $0$ if either of the pulled arms is optimal.  This setting is more appropriate for recommender systems, in which we offer the user a pair of items, and she selects the one that is preferred.  $0$ regret is incurred as long as the best item is made available.  While \cite{yue2012k} introduced weak regret, an algorithm with regret bounds first appeared in \cite{chen2017dueling}, which proposed the \textit{Winner Stays} (WS) algorithm that achieves $O(N\log(N))$ cumulative binary weak regret when arms have a total order and $O(N^2)$ in the Cordorcet winner setting.  These bounds on binary weak regret have corresponding bounds on utility-based weak regret inflated by the difference in utility between the best and worst arms.

% recommender systems such as \textit{UberEat}.  Weak regret in dueling bandits was first proposed in \citep{yue2012k} but it was not as explored as the strong regret. Based on our best knowledge, \citet{chen2017dueling} is the first paper studies dueling bandit with weak regret. \citet{chen2017dueling} proposes a simple algorithm called \textit{Winner Stays} (WS), which assigns a score to each of the arm and pulls the arm which has the top two largest score at each time. They prove that WS can achieve $O(N\log(N))$ cumulative weak regret under the total order assumption and $O(N^2)$ cumulative weak regret in the Condorcet winner setting. However, WS completely ignores the dependency among arms, which could help us learn the user's preference faster. In this paper, our proposed strategy takes this dependency into consideration and achieves better result.

We consider utility-based weak regret, in the total order setting, when the total order is induced by a utility which is in turn a function of observable arm features, an unknown latent preference vector, and a known utility function.  
This framework includes the commonly used logit or Bradley-Terry \citep{revelt1998mixed,yue2012k} and  probit models \citep{franses2002econometric}.  
We provide an algorithm, Comparing with the Best (CTB) that has expected cumulative utility-based weak regret that is constant in $T$, and that leverages the dependence between preferences over arms induced by the arm features and utility function to provide excellent empirical performance when prior information is available.
While our regret bound's dependence on $N$ is looser than \cite{chen2017dueling} (our dependence is $2^N$ in the worst case, and is $N^{2d}$ when the utility function is linear over a $d$-dimensional space of preferences and arm features), our algorithm is more flexible in its ability to problem structure induced by the feature vectors, and outperforms it empirically by a substantial margin when $N$ is small enough to allow computation that fully takes advantage of this problem structure.  

Our exploitation of arm features is similar in spirit to work in the traditional (cardinal) multi-armed bandit setting on linear bandits \citep{rusmevichientong2010linearly, abbasi2011improved}.

% The exploration vs exploitation trade off we addressed in this paper is studied more thoroughly in the multi-armed bandits setting \citep{bubeck2012regret, auer2002finite}. The closest related work in the multi-armed bandits setting with ours is the best arm identification problem \citep{audibert2010best,bubeck2013multiple}. However, in their settings, they are trying to identify the best arm within a fixed budget which is different from ours. 

% Our paper is also closely related to the linear bandits problem studied in \citep{rusmevichientong2010linearly}, in which each arm corresponds to a covariate and we could explore more efficiently based on these covariate.

The paper is structured as follows. In section~\ref{probForm}, we formulate our problem.
In section~\ref{Methods}, we introduce {\it Comparing The Best} (\CTB) which we show in section~\ref{results} has \CTB\ constant expected cumulative regret.  In section~\ref{Imple}, we discuss a efficient implementation method for a specific class of prior information. In section~\ref{sec:Bayes}, we provide a Bayesian interpretation for \CTB. In section~\ref{sec:exp}, we compare \CTB\ with three benchmarks using simulated datasets, in which \CTB\ outperforms all benchmarks considered.

\newcommand{\prefdim}{d'} % The dimension of the preference vector
\newcommand{\armdim}{d} % The dimension of the arm feature vector

\section{Problem Formulation}
\label{probForm}

There are $N\geq 2$ arms, and each arm $i$ has an observable and distinct $\armdim$-dimensional feature vector $A_{i}$.
Preferences between pairs of arms $i,j$ are described by fixed but unknown probabilities $p_{i,j}$, where $p_{i,j} = 1 - p_{j,i}$ and $p_{i,j} \ne 0.5$ when $i \ne j$.
We denote $p=\min_{i<j} \max(p_{i,j},p_{j,i})$. By construction, $p>0.5$.

At each time t, we pull two arms $X_{t,0}$ and $X_{t,1}$ (this act is called a ``duel'') and we observe feedback $Y_{t}\in \{0,1\}$ indicating the winning arm: $Y_{t}=0$ indicates arm $X_{t,0}$ won and $Y_{t}=1$ indicates arm $X_{t,1}$ won.
Conditioned on the arms pulled and the history (the arms pulled and the identity of the winner at times $t'<t$), $Y_t$ is equal to $0$ with probability $p_{i,j}$.

We suppose that the arms have a total order, i.e., that there exists an ordering of the arms such that $p_{i,j} > 0.5$ if and only if arm $i$ is before arm $j$ in this order.  
Moreover, we suppose this ordering is determined by a utility associated with each arm, $u(\theta, A_i)$, where $u$ is a known utility function and 
$\theta \in \mathbb{R}^{\prefdim}$ is an unknown preference vector.
In particular, $p_{i,j} > 0.5$ if and only if $u(\theta, A_{i})>u(\theta,A_{j})$. 
The assumption that the total order be determined by $u(\theta,A_i)$ is without loss of generality if we are willing to select $\prefdim$ to be sufficiently large and $u$ to allow sufficient flexibility, although one may also choose a smaller $\prefdim$ and a less flexible $u$ with the goal of obtaining smaller regret (described below) when these more restrictive modeling assumptions hold. We assume without loss of generality that the indices correspond to their ordering by utility, so $u(\theta, A_{1}) >  u(\theta, A_{2}) > \cdots > u(\theta,A_{N})$. 

Several commonly used discrete choice models fall within this framework. For example, our framework includes the logit or Bradley-Terry model \citep{revelt1998mixed,yue2012k}, in which $\prefdim=\armdim$, the utility function is $u(\theta,A_{i})=\theta\cdot A_{i}$ and 
$p_{i,j} =\frac{\exp(u(\theta,A_{i}))}{\exp(u(\theta,A_{i})+u(\theta,A_{j}))}$.
Our framework also includes the probit model \citep{franses2002econometric} in which $\prefdim = \armdim$ and the utility function is the inner product as with the logit model, but $p_{i,j}=\Phi(u(\theta,A_{i})-u(\theta,A_{j}))$ where $\Phi(\cdot)$ is the standard normal cdf.

We define the utility-based weak regret $r(t)$ (henceforce referred to simply as the regret) at time $t$ as $r(t)=u(\theta,A_{1})-\max\{u(\theta,A_{X_{t,0}}), u(\theta,A_{X_{t,1}})\}$, which is the difference in utility between the best arm overall and the best arm available to the user from those offered.  
The cumulative regret up to time $T$ is $R(T)=\sum_{t=1}^{T}r(t)$.
We measure the quality of an algorithm by its expected cumulative regret. 

We now develop an algorithm \CTB, and show it has constant expected cumulative regret.

\section{The \textit{Comparing The Best} (CTB) Algorithm}
\label{Methods}

In this section we propose an algorithm {\it Comparing The Best} (\CTB) for this problem setting. This algorithm is based on the idea of ``cells'', which correspond to possible orderings of the arms by utility. It maintains a score for each cell, either explicitly or implicitly, which it initializes using optional prior information, and updates with the results from each duel. 

We present a general version of \CTB\ in this section that admits any prior information and explicitly maintains a score for each cell. Because the number of cells is exponential in the number of arms, explicitly maintaining scores for each cell is computationally infeasible for large problems. Thus, after presenting our theoretical results for the general \CTB\ algorithm in section~\ref{results}, we present a computationally efficient implementation of our algorithm in section~\ref{Imple} that can be used when the prior information can be expressed in terms of an initial score for each pair of arms. Although we present our algorithm in a frequentist setting, we show in section~\ref{sec:Bayes} that the scores used for each cell correspond to a Bayesian posterior on the value of $\theta$, and \CTB\ has a natural Bayesian interpretation.

To define \CTB, we first define some terminology and notation: {\it winning spaces}, {\it cells}, a {\it score}, and the best arm corresponding to a cell.  We begin with winning spaces.

\begin{definition}
Each pair of arms $i,j$ defines a {\it winning space} $H_{i,j} := \{X\in \mathbb{R}^{d}: u(X,A_{i}) \geq u(X, A_{j})\}$. 
\end{definition}
When $\theta\in H_{i,j}$, arm $i$ is preferred over arm $j$. 
% The union of $H_{i,j}$ and $H_{j,i}$ is $\mathbb{R}^{d}$. 
We use the phrases ``arm $A_i$ wins over arm $A_j$ in a duel'', and ``winning space $H_{i,j}$ wins the duel'' interchangeably.
% When the utility function is linear, this winning space is a half space. 

Each pair of arm determines two winning spaces and all winning spaces partition the space $\mathbb{R}^{d}$ into cells, where each cell is an intersection of winning spaces. To define notation to support working with cells, we first define  $H_{i,j}(k)=H_{i,j}$ when $k=0$ and $H_{i,j}(k)=H_{j,i}$ when $k=1$. For a binary vector $V$, we let $V[k]$ denote the $k^{th}$ element of $V$. Then, we have the following definition.

\begin{definition}
The {\it cell} C corresponding to a length $\frac{N(N-1)}{2}$ binary vector V is 
\begin{align*}
    C(V):=\cap_{i<j}H_{i,j}\left(V\left[\frac12 (2N-i)(i-1)+j-i\right]\right).
\end{align*}
\end{definition}

%In this definition, we use the $\left[\frac12 (2N-i)(i-1)}+j-i\right]^{th}$ element of $V$ to denote whether $C$ belongs to $H_{i,j}$ or $H_{j,i}$, for $i<j$. 

% \pfcomment{Can we change notation so that we don't need the new notation introduced by the below paragraph?  We can have $m_V$ instead of $\m{s}$, $J_V$ instead of $J_s$, and have $\mathcal{C}$ be the set of cells.  We can use $0$ to indicate the correct $V$ instead of writing out $[0,...,0]$, so that $m_1$ becomes $m_0$.
% If we need to sum over all of the non-empty cells, or the cells that have finite $m_V$, we can let $\mathcal{C}$ refer to this set.  Then $M = |\mathcal{C}|$.
% }

We assign binary vectors indexing cells, all of length $\frac{N(N-1)}{2}$, to integers lexicographically.
Let $V_{k}$ denote the $k^{th}$ such binary vector, let $M=2^N$ denote the number of cells, and let $C_i = C(V_i)$.
With this definition, 
$C_{1} = C(V_{1}) =C([0,0,\cdots,0])$ and thus $C_{1}=\cap_{i<j}H_{i,j}$ and $\theta \in C_{1}$. 
Some cells $C_i$ may be empty. We call these {\it empty cells}. 
Let $J_{k}=\{(i,j)|C_{k}\subseteq H_{i,j}\}$, which is the collection of indices of the winning spaces that contains $C_{k}$. 

Figure~\ref{illustration} illustrates winning spaces and cells.
% \pfcomment{If you remake this figure, Bangrui, let's avoid having the letters and numbers touch the lines.  Also, change the lines so that $C_5$ fits --- it adds complexity to the visual presentation by having an arrow in there.}

\begin{figure}[!h]
    \centering
    \includegraphics[scale=0.5]{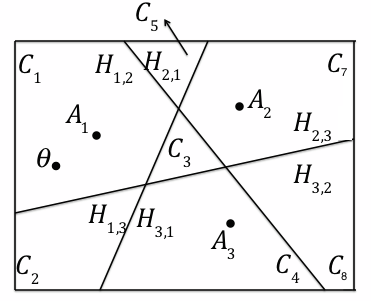}
    \caption{Illustration of winning spaces and cells. The index of the cell and its corresponding binary vectors are: $C_{1}$ and $(0,0,0)$; $C_{2}$ and $(0,0,1)$; $C_{3}$ and $(0,1,0)$; $C_{4}$ and $(0,1,1)$; $C_{5}$ and $(1,0,0)$; $C_{7}$ and $(1,1,0)$; $C_8$ and $(1,1,1)$. In this case, cell $C_6$ is an empty cell since the intersection of $H_{2,1}$, $H_{1,3}$ and $H_{3,2}$ is empty.}
    \label{illustration}
\end{figure}

\newcommand{\m}[1]{m_{#1}(0)}
\newcommand{\mi}{\m{i}}

We define a score $m_i(t)$ associated with each cell $C_i$ at time $t$.  Later in section~\ref{sec:Bayes} we will interpret this score as a monotone transformation of the posterior probability that $\theta$ is in this cell.  This score will be initialized to some value $\mi$, discussed below, and then will be incremented each time a winning space containing $C_i$ wins a duel.  That is,

\begin{align}
m_{i}(t)&=\mi +\sum_{k=1}^{t}\mathbbm{1}\{C_{i}\subseteq H_{X_{k,1},X_{k,2}}(Y_k)\}.
\label{update}
\end{align}

Each cell $C_{i}$ assigns a preference order to the arms.  Let $B(i)$ be the arm that would be best if $\theta$ were in $C_i$. More formally, $B(i)$ is the unique $j$ such that $C_{i}\subseteq H_{j,k}$, $\forall k\neq j$. Since $\theta \in C_{1}$, we know $B(1)=1$.

With this notation, we now define the {\it Comparing The Best} (\CTB) algorithm in Algorithm~\ref{algo1}.  \CTB\ pulls the arm that is best according to the cell with the highest score $m_i(t)$, and the arm that is best according to the cell with the highest score among those that have different best arm from the first arm chosen.
If we interpret $m_i(t)$ as being a monotone transformation of the posterior probability that $\theta \in C_i$, then we are selecting arms by selecting two cells that have different best arms, and are together most likely to contain $\theta$.

\begin{algorithm}[H]
 \For{t $\leq$ T}{
  Step 1: Pick $X_{t,0}=B(\argmax_{i}m_{i}(t))$, breaking ties arbitrarily\\
  Step 2: Pick $X_{t,1}=B\left(\argmax_{i:B(i)\neq X_{t,0}}m_{i}(t)\right)$, breaking ties arbitrarily\\
  Step 3: Observe the noisy feedback $Y_{t}$ and update $m_{i}(t)$ using Equation~\eqref{update}\\
  Step 4: t=t+1
 }
 \caption{{\it Comparing The Best} (\CTB)}
 \label{algo1}
\end{algorithm}

\paragraph{Choice of $\mi$:}
Here we offer guidance on the choice of $\mi$, which is left general in the description of \CTB\ to allow the user the flexibility to influence the arms pulled with prior information about the value of $\theta$, and to trade off regret against CTB's computational performance. In doing so, there are four considerations:

First, by setting $\mi$ larger for those cells that the user believes are more likely to contain $\theta$, the user encourages \CTB\ to select those cells more often.  If the user correctly sets $\mi$ larger for the cell that contains $\theta$, this tends to pull the best arm more often and decrease regret. We show in section~\ref{sec:Bayes} that $\mi$ can be interpreted in terms of the prior probability that $\theta \in C_i$, and one can leverage this relationship to convert prior information on $\theta$ into values for $\mi$. 

Second, by setting $\mi$ to be $-\infty$ for those cells that user is certain do not contain $\theta$, she can lead \CTB\ to never select those cells.  One may safely do this for empty cells, in which model assumptions imply $\theta$ cannot reside.  Doing this for other cells is dangerous, as setting cell $\m{1}$ to $-\infty$ can cause \CTB\ to have linear regret.

Third, in the absence of prior information, one may simply set $\mi=0$ for all cells that may contain $\theta$. We show in the next section show that as long as $\m{1}>-\infty$, the expected cumulative regret is finite.

Fourth, there is a computational aspect to setting $\mi$.  We show below in section~\ref{Imple} that if each $\mi$ can be written as a sum across pairs of arms of a score associated with each pair, then we can implement \CTB\ in a computationally efficient manner that scales to many arms.  In contrast, if one sets $\mi$ without enforcing structure, the computation required to implement Algorithm~\ref{algo1} grows exponentially with the number of arms.

% By setting $\mi$ to be larger for some cells and smaller for others, one can encourage \CTB\ to pull more often those arms believed to be the best by cells with larger $\mi$.  If one has prior information that $\theta$ is more likely to be in a particular cell, it may reduce regret to 

With these considerations in mind, we propose 3 specific ways to set $\mi$, and evaluate them in numerical experiments:
\begin{itemize}
\item For situations with loose computational requirements or few arms, and no prior information, we recommend setting $m_i = 0$ for all non-empty cells and $m_i=-\infty$ for all empty cells.  We call this \CTBone.
\item For situations with strict computational requirements and no prior information, we recommend setting $m_i = 0$ for all cells.  Then \CTB\ can be implemented using the efficient method described in section~\ref{Imple}.  We call this \CTBtwo.
\item For situations with loose computational requirements or few arms, and strong prior information, we recommend setting $m_i$ from the prior according to the method described in section~\ref{sec:Bayes}.  We call this \CTBthree. 
\end{itemize}

\section{Theoretical Results}
\label{results}
In this section, we prove the expected cumulative regret of \CTB\ is bounded by a constant. 
The main idea behind our proof is to show that for each cell $C_{i}$ with $B(i)\neq 1$, $\mathbb{E}[\sum_{t=0}^{\infty}1\{m_{i}(t)\geq m_{1}(t)\}]$ is bounded by a constant. We show this in turn by relating $m_{1}(t)-m_{i}(t)$ to a random walk with a larger probability of increasing than of decreasing. The following lemma, whose proof is in the supplement, allows us to bound the number of times this stochastic process takes values less a constant.

\begin{lemma}
\label{basic}
Let $p \in (0.5,1]$.
Suppose $Z(t)$ is a stochastic process with filtration $\mathcal{F}_t$, $Z(0)=0$  
and $P(Z(t+1)=Z(t)+1|\mathcal{F}_t)\geq p$,
then we have $\mathbb{E}\left[\sum_{t=0}^{\infty} \mathbbm{1} \{ Z(t)\leq S\}\right] \leq \frac{p+S(2p-1)}{(2p-1)^{2}}$ for $S\in \mathbb{N}$.
\end{lemma}

We now proceed with the larger proof by defining
\begin{align}
q_{i,j}(t)&=\sum_{k=1}^{t}\mathbbm{1}\{X_{k,0}=i, X_{k,1}=j, Y_{k}=0\} +\sum_{k=1}^{t}\mathbbm{1}\{X_{k,0}=j,X_{k,1}=i,Y_{k}=1\}, \label{update2}
\end{align}
which is the number of times up to time $t$ that arm $i$ beats arm $j$ in a duel. 
Then we can rewrite $m_{i}(t)$ in terms of $q_{i,j}(t)$ as,
\begin{align}
m_{k}(t)=\m{k}+\sum_{(i,j)\in J_{k}}q_{i,j}(t).
\label{reconstruct}
\end{align}

The definition of $C_{1}$ implies $J_{1}=\{(i,j),\forall i<j\}$ and $m_{1}(t) = \m{1}+\sum_{i<j}q_{i,j}(t)$. Let $N_{i,j}(t) = q_{i,j}(t) + q_{j,i}(t)$ denote the number of times we have pulled arms $i$ and $j$. 
The next lemma shows $\mathbb{E}[N_{i,j}(t)]$ is bounded by a constant for $1<i<j$.

\begin{lemma}
\label{mainLemma}
For $1<i<j$, if $\m{1} > -\infty$, we have $\mathbb{E}[N_{i,j}(t)]\leq M'\frac{p-\Delta(2p-1)}{(2p-1)^{2}}$, where $M'$ is the number of cells $i$ with $\m{i} > -\infty$, and $\Delta=\min_{s=1,\cdots M}\{\m{1}-\m{s}\}\leq 0$. 
\end{lemma}
\begin{proof}
Let $1<i<j$. 
Let $D_{i,j}(t)$ be an indicator function equal to $1$ if and only if we pull arms $i$ and $j$ at time $t$. 
Given that we pull arm $i$, we can only also pull arm $j$ when there is a cell $C_{s}$ under which $j$ is the best arm and for which $m_{s}(t) \ge m_{1}(t)$. 
Moreover, under the assumption that $\m{1}>-\infty$, $m_{s}(t) \ge m_{1}(t)$ is only possible if $\m{s}>-\infty$.
Thus, $D_{i,j}(t) = 1$ implies 
$\max_{s:B(s)=j,\m{s}>-\infty} m_{s}(t)\geq m_{1}(t)$.
Adopting the convention here and in the rest of the proof that maxima and sums over sets of cells are taken only over those cells with $\m{s} > -\infty$, we have 
\begin{align}
D_{i,j}(t)
&=  D_{i,j}(t)\cdot \mathbbm{1}\left\{\max_{s:B(s)=j}m_{s}(t)\geq m_{1}(t)\right\} \nonumber \\
&\leq D_{i,j}(t)\sum_{s:B(s)=j}\mathbbm{1}\{m_{s}(t)\geq  m_{1}(t)\} \nonumber \\
&= D_{i,j}(t)\sum_{s:B(s)=j}\mathbbm{1}\left\{\sum_{(i^{'},j^{'})\in J_{s}}q_{i^{'},j^{'}}(t)+\m{s}\geq  \sum_{(i^{'},j^{'})\in J_{1}}q_{i^{'},j^{'}}(t)+\m{1}\right\} \nonumber \\
&=D_{i,j}(t)\sum_{s:B(s)=j}\mathbbm{1}\left\{\sum_{(i^{'},j^{'})\in J_{s}\setminus J_{1}}q_{i^{'},j^{'}}(t)+\m{s}\geq  \sum_{(i^{'},j^{'})\in J_{1}\setminus J_{s}}q_{i^{'},j^{'}}(t)+\m{1}\right\} \nonumber \\
&=D_{i,j}(t)\sum_{s:B(s)=j}\mathbbm{1}\left\{\sum_{(i^{'},j^{'})\in J_{s}\setminus J_{1}} q_{i^{'},j^{'}}(t)-q_{j^{'},i^{'}}(t) \geq  \m{1}-\m{s}\right\} \nonumber \\
&\leq D_{i,j}(t)\sum_{s:B(s)=j}\mathbbm{1}\left\{\sum_{(i^{'},j^{'})\in J_{s}\setminus J_{1}} q_{i^{'},j^{'}}(t)-q_{j^{'},i^{'}}(t) \geq  \Delta\right\}, \nonumber
\end{align}
 %where the fourth equation holds because $J_{1}=\{(i^{'},j^{'}): i^{'}<j^{'}\}$ and thus $(i^{'},j^{'})\in J_{s}\setminus J_{1}\iff (j^{'},i^{'})\in J_{1}\setminus J_{s}$.  \pfcomment{Don't you also need something about $J_s$ for for the fourth equation to hold?  For example, if $J_s = \emptyset$ then what you say isn't true.  I think you need that if $i',j'$ is in $J_s$ then $j',i'$ is not in $J_s$.  That is also what you are using about $J_1$.}
where the fourth equation holds because $J_s$ has the property that $(i',j') \in J_s \iff (j',i') \notin J_s$, and similarly for $J_1$.  Thus,
$(i^{'},j^{'})\in J_{s}\setminus J_{1}
\iff i',j' \in J_s\text{ and } i',j' \notin J_1
\iff j',i' \notin J_s\text{ and } j',i' \in J_1
\iff (j^{'},i^{'})\in J_{1}\setminus J_{s}$.

Thus, we have
\begin{align}
&N_{i,j}(t) =\sum_{k=1}^{t}D_{i,j}(k) \leq
D_{i,j}(k)
\sum_{s:B(s)=j}\sum_{k=1}^{t}\mathbbm{1}\left\{\sum_{(i^{'},j^{'})\in J_{s}\setminus J_{1}} q_{i^{'},j^{'}}(k)-q_{j^{'},i^{'}}(k) \geq  \Delta \right\}.\nonumber  
\end{align}

Fix an $s$ with $B(s)=j$ and 
let $Z(k)=\sum_{(i^{'},j^{'})\in J_{s}\setminus J_{1}} q_{i^{'},j^{'}}(k)-q_{j^{'},i^{'}}(k)$, so that 
\begin{equation*}
N_{i,j}(t) \leq
\sum_{s:B(s)=j}
\sum_{k=1}^{t} D_{i,j}(k) \cdot \mathbbm{1}\left\{Z(k) \geq \Delta \right\}.
\end{equation*}

We observe that $Z(k)$ is like a random walk, except that changes in only some time periods.  We now describe the conditional distribution of $Z(k+1)$ given the history up to time $k$.  Later, we will refer to the $\sigma$-algebra generated by this history as $mathcal{H}_k$.
\begin{itemize}
    \item 
    If the arms $X_{k,0}$, $X_{k,1}$ that we pull satisfy 
    $(X_{k,0},X_{k,1})\in J_{s}\setminus J_{1}$,
    then $Z(k+1) \in \{Z(k)-1,Z(k)+1\}$ 
    and the conditional probability that 
    $Z(k+1)=Z(k)-1$ is $p_{X_{k,1},X_{k,0}}\geq p$.  
    This lower bound holds because $(X_{k,0},X_{k,1})\notin J_{1}$ implies  
  $X_{k,1} < X_{k,0}$.
    \item 
    Similarly, if 
    $(X_{k,1},X_{k,0})\in J_{s}\setminus J_{1}$,
    then $Z(k+1) \in \{Z(k)-1,Z(k)+1\}$ as before, 
    and the conditional probability that 
    $Z(k+1)=Z(k)-1$ is $p_{X_{k,0},X_{k,1}}\geq p$,   
   because    
  $(X_{k,1},X_{k,0})\notin J_{1}$ implies  
  $X_{k,0} < X_{k,1}$.
    \item Otherwise, if neither 
    $(X_{k,0},X_{k,1})$ 
    nor 
    $(X_{k,1},X_{k,0})$ is in $J_{s}\setminus J_{1}$,
    then $Z(k+1)=Z(k)$.
    \item The definition of $J_1$ prevents having both $(X_{k,0},X_{k,1})$ and  $(X_{k,1},X_{k,0})$  in $J_{s}\setminus J_{1}$.
\end{itemize}

When $D_{i,j}(k) = 1$, so that we pull arms $i$ and $j$ (either $X_{t,0}=i$ and $X_{t,1}=j$ or vice versa) we will be in one of the first two cases,
because $B(s)=j$ implies cell $s$ considers $j$ to be the best arm, and so $(j,i) \in J_s$, and $i<j$ implies $(j,i) \notin J_1$. Thus, $D_{i,j}(k) = 1$ implies $Z(k+1) \ne Z(k)$, and we have
\begin{equation*}
N_{i,j}(t) \leq
\sum_{s:B(s)=j}
\sum_{k=1}^{t} \mathbbm{1}\left\{Z(k+1) \ne Z(k), Z(k) \geq \Delta \right\}.
\end{equation*}

We will perform a random time change to study the dynamics over only those time periods where $Z(k)$ changes. Define $\tau_0 = 0$, $\tau_{m}=\min_{k}\{k>\tau_{m-1},Z(k)\neq Z(k+1)\}$.
Because the event $Z(k) \neq Z(k+1)$ is measurable given the history at time $k$, $\mathcal{H}_k$, as described in the dynamics of $Z(\cdot)$ above, each $\tau_m$ is a stopping time.
Define $\zeta=\inf\{m:\tau_{m}=\infty\}$,
which is the lifetime of the random change of time.
We have,
\begin{equation}
N_{i,j}(t) \leq
\sum_{s:B(s)=j}
\sum_{m=1}^{\zeta-1} \mathbbm{1}\left\{Z(\tau_m) \geq \Delta \right\}.
\label{eq:proof1}
\end{equation}

We let $W(m)=Z(\tau_m)$ for $m < \zeta$ (i.e., $m$ with $\tau_m < \infty$), and $W(m) = W(m-1) + \epsilon_m$ for $m\ge \zeta$, where $\epsilon_m$ are iid random variables taking value $-1$ with probability $p$ and value $1$ with probability $1-p$. 
Observe that $\zeta$ is measurable with respect to $\mathcal{H}_\infty$, so that the event $m < \zeta$ is measurable with respect to $\mathcal{H}_{\tau_m}$.  We define an augmented filtration, letting $\mathcal{F}_m$ to be the $\sigma$-algebra generated by 
$\mathcal{H}_{\tau_{\min(m,\zeta)}}$ and $(\epsilon_{m'} : m' \le m)$.
With this construction, $W(m+1) - W(m) \in \{-1,+1\}$ and 
$P\left(W(m+1) = W(m)-1 | \mathcal{F}_m \right) \ge p$.  
Thus, by Lemma~1, 
\begin{equation*}
\sum_{m=1}^{\zeta} \mathbbm{1}\left\{Z(\tau_m) \geq \Delta \right\}
= \sum_{m=1}^{\zeta} \mathbbm{1}\left\{W(m) \geq \Delta \right\}
\le \sum_{m=1}^{\infty} \mathbbm{1}\left\{W(m) \geq \Delta \right\}
\le \frac{p-\Delta(2p-1)}{(2p-1)^{2}}.
\end{equation*}

Combining this with \eqref{eq:proof1} and using the fact that the number of cells with $\m{s} > -\infty$, $M'$, bounds the sum over $s$, we obtain our result.\qedhere

\end{proof}

Based on Lemma~\ref{mainLemma} and a union bound, we obtain our main theorem:

\begin{theorem}
\label{thm:1}
Let $\Lambda=u(\theta,A_{1})-u(\theta,A_{N})$. 
If $\m{1}>-\infty$, \CTB's expected cumulative regret is bounded by $\frac{(N-1)(N-2)}{2}M'\frac{p-\Delta(2p-1)}{(2p-1)^{2}}\Lambda$.
\end{theorem}

In general, $M'$ can be as large as $2^N$.  However, as discussed above, we may set $\m{i}=-\infty$ for all the empty cells and assign finite $\m{i}$ to empty cells (\CTBone). In this setting, since each cell assigns a ranking over arms and different cells give different rankings, we can bound $M'$ by the number of permutations of $N$ arms, $N!$. Moreover, when the utility function is linear and $\prefdim=\armdim$, results in \cite{jamieson2011active} show $M'$ is $O(N^{2\prefdim})$.

\section{Computation for Decomposable $m_i$}
\label{Imple}

\CTB\ achieves a constant expected cumulative regret. However, a naive implementation of Algorithm~\ref{algo1} requires a great deal of memory to store $m_{i}(t)$ for each cell, which makes it computationally challenging for problems with many arms. In this section, we consider a special case of \CTB\ where $m_i(0)$ can be expressed in terms of an initial score for each pair of arms.  Specifically, we suppose that there exists a $r_{i,j}$ such that 
\begin{equation}
m_k(0)=\sum_{(i,j)\in J_{k}}r_{i,j} \quad \forall k. 
\label{eq:decompose}
\end{equation}
Here $r_{i,j}$ can be interpreted as a prior indicating the extent to which we believe that arm $i$ is preferred over arm $j$. In this special case, we describe an efficient computation method that scales to  problems with many arms.

Instead of storing $m_{i}(t)$, this method
stores $r_{i,j}$ and $q_{i,j}(t)$ and uses them to reconstruct $m_{i}(t)$ with Equation~\ref{reconstruct}.
Then, Steps 1 and 2 in Algorithm~\ref{algo1} are written as optimization problems in which $m_i(t)$ is replaced by this expression in terms of $q_{i,j}(t)$ and $r_{i,j}$. Toward this end, let $e_{i,j}$ denote a binary variable that will take value $e_{i,j}=1$ if we are to select a cell in $H_{i,j}$ and $0$ otherwise.  Then, based on Equation~\ref{reconstruct}, maximizing $m_{i}(t)$ is equivalent to maximizing $\sum_{i,j:i\neq j}e_{i,j}\times (q_{i,j}(t)+r_{i,j})$.

To find the best arm suggested by $\argmax_{i}m_{i}(t)$ in Step~1, and suggested by a similar argmax in Step~2, it is sufficient to find $\max_{i:B(i)=k} m_{i}(t)$ for each arm $k$.  This is the cell with largest $m_{i}(t)$ among those that believe $k$ is best. This problem is:
\begin{align}
\label{IP}
\begin{split}
\text{maximize } \displaystyle &\sum_{i,j:i\neq j} e_{i,j}\times (q_{i,j}(t)+r_{i,j}) \\
\text{subject to } \displaystyle
&e_{k,j}=1,\ \forall j\neq k\\
&e_{i,j}+e_{j,i}=1,\ i,j=1 ,..., N, i\neq j\\
&e_{i,j} \in \{0,1\},\ \forall i\neq j
\end{split}
\end{align}

There are three conditions in Equation~\ref{IP}. The first condition is $e_{k,j}=1$ $\forall j\neq k$, which means cell $C_{\ell}$ that satisfies the first condition must lie in the winning space $H_{k,j}$, $\forall j\neq k$. In other words, $C_{\ell}$ ranks arm $A_{k}$ better than any others and thus $B(\ell)=k$. The second and third condition together guarantee that cell $C_{\ell}$ either belongs to $H_{i,j}$ or $H_{j,i}$. 

Though Equation~\ref{IP} is an integer linear programming problem, which are usually computationally challenging, it is in fact easy to solve: the maximum value of this problem is reached when $e_{i,j}=1$ if $r_{i,j}+q_{i,j}(t) > q_{j,i}(t)+r_{j,i}$ for all $i\neq j$, $e_{i,j}=0$ if this strict inequality is reversed, and breaking ties arbitrarily between the solutions $(e_{i,j}=1, e_{i,j}=0)$ and 
$(e_{i,j}=0, e_{i,j}=1)$ for those $i,j$ with equality. 

Denote the maximum value of this problem at time $t$ as $f(k,t)$. After knowing $f(k,t)=\max_{B(i)=k} m_{i}(t)$, finding the arm with largest $m_{i}(t)$ in Step 1 is equivalent to finding $\argmax_{k}f(k,t)$.  Finding the arm with large $m_{i}(t)$ among those with a different best arm than $X_{t,0}$ in Step~2 is equivalent to finding $\argmax_{k \ne X_{t,0}}f(k,t)$.

For general values of $m_i(0)$ that do not satisfy \eqref{eq:decompose}, finding the largest $m_i(t)$ is computationally challenging. However, in applications, instead of setting $m_i(0)$ directly, we may have some prior information about the probability that the user prefers arm $i$ over arm $j$. This information can be used to construct $r_{i,j}$ since \CTB\ guarantees constant regret regardless of the values that $m_i(0)$ take.

\section{Bayesian Interpretation}
\label{sec:Bayes}
Although our problem is formulated in a frequentist setting, we show here that \CTB\ has a Bayesian interpretation. In this section, we construct a Bayesian posterior on $\theta$ given a prior and given an assumption that $p_{i,j}=q>0.5$ for all $i<j$, where $q$ may be the same or different from $p$, and $p_{i,j}$ may or may not be constant across $i,j$ in reality. 

We put a prior distribution $p_0$ on $\theta$, which induces a prior on the identity of the cell containing $\theta$.  The prior probability that $\theta$ is in cell $i$ is written $p_0(C_i)$, and is obtained by integrating $p_0$ over $C_i$. Let $p_t(C_i)$ indicate the posterior probability that $\theta$ is in cell $C_i$, at time $t$, given $p_{i,j}=q$ for all $i<j$.  The following pair of lemmas give recursive and non-recursive expressions for $p_t$.

\begin{lemma}
\label{lemma:bayes}
For compactness of notation, let $i=X_{t,0}$ and $j=X_{t,1}$.
Then the posterior distribution $p_{t+1}$ is,
\[
p_{t+1}(x) =
  \begin{cases}
    \frac{p_{t}(x)q}{p_{t}(H_{i,j}(Y_k))q+(1-p_{t}(H_{i,j}(Y_k)))(1-q)}     & \quad \text{if } x\in H_{i,j}(Y_t)\\
    \frac{p_{t}(x)(1-q)}{p_{t}(H_{i,j}(Y_k))q+(1-p_{t}(H_{i,j}(Y_k)))(1-q)}   & \quad \text{if } x\notin H_{i,j}(Y_t)\\
  \end{cases}
\]
\end{lemma}

Based on this lemma, we can rewrite the posterior distribution in terms of $m_{i}(t)-m_{i}(0)$.

\begin{lemma}
\label{posterior}
For each cell $C_{i}$, the posterior distribution after t comparison is 
\begin{align}
    p_{t}(C_{i}) \propto p_{0}(C_{i})q^{m_{i}(t)-m_{i}(0)}(1-q)^{t-m_{i}(t)+m_{i}(0)}. \nonumber
\end{align}
\end{lemma}

We leave the proof of both Lemmas to the appendix.
Lemma~\ref{posterior} allows us to rewrite $p_{t}(C_{i})$ as 
\begin{align}
        p_{t}(C_{i}) \propto p_{0}(C_{i})q^{m_{i}(t)-m_{i}(0)}(1-q)^{t-m_{i}(t)+m_{i}(0)} \nonumber \\
        \propto p_{0}(C_{i})(\frac{q}{1-q})^{m_{i}(t)-m_{i}(0)}. \nonumber
\end{align}
Thus, choosing the cell to maximize the posterior probability is equivalent to choosing the cell to maximize
$\log(p_{0}(C_{i}))+(m_{i}(t)-m_{i}(0))\log\left(\frac{q}{1-q}\right).$
Thus, if 
\begin{equation}
\label{eq:mi0}
m_i(0)=\log(p_{0}(C_{i})) \big/ \log\left(\frac{q}{1-q}\right),
\end{equation}
then maximizing the posterior probability that $\theta$ is in $C_i$ is equivalent to maximizing $m_i(t)$,
the first cell selected by \CTB\ is the cell with the largest posterior probability of containing $\theta$, and the second cell selected is the largest among those with a different best arm from the first.

Thus, if one has prior information about the location of $\theta$ and an estimate $q$ of a typical value of $p_{ij}$, then a natural way to set $m_i(0)$ is via \eqref{eq:mi0}.
In addition, since $p_{0}(C_{i})=0$ for empty cells, following \eqref{eq:mi0} also sets $m_i(0)=-\infty$ for these cells as discussed before. 

\section{Numerical Experiments}
\label{sec:exp}

In this section, we compare the three variants of \CTB\ described in section~\ref{Methods}, CTB-1, CTB-2, and CTB-3, with three benchmarks: Thompson Sampling, Relative Upper Confidence Bound (RUCB) and Winner-Stays (WS). 

\begin{itemize}
\item Thompson sampling uses a posterior distribution over $\theta$ computed by beginning with a prior distribution on the location of $\theta$, and updating it using Bayes rule and knowledge of $p_{i,j}$. At time t, it generates $\theta_{t}$ from this posterior distribution $p_{t}$ and pulls the two arms that $\theta_{t}$ ranks as best and second best. In our implementation, we track the prior/posterior explicitly by storing a probability for each cell. We emphasize that Thompson sampling as we consider it here requires knowledge of $p_{i,j}$ which is not typically not available.
\item RUCB is as described in \cite{zoghi2014relative}. We choose it as our benchmark over other algorithms designed for strong regret from the literature because it works well relative to other algorithms designed for strong regret in previous literature when a Condorcet winner exists, and existence of a Condorcet winner is a consequence of our total order assumption. Though there are algorithms that outperform RUCB in some settings such as CCB and SCB \citep{zoghi2015copeland}, they typically work better when a Condorcet winner does not exist.
\item WS is as described in \cite{chen2017dueling}, and is selectetd because it is designed for the weak regret setting.  In our plots, WS-W is the variant of WS designed specifically for weak regret.
\end{itemize}

We consider two experimental settings described below, with results pictured in Figure~\ref{fig:result2}. 

\begin{figure*}[!h]
    \centering
    \begin{subfigure}[t]{0.5\textwidth}
        \centering
        \includegraphics[width=1\textwidth]{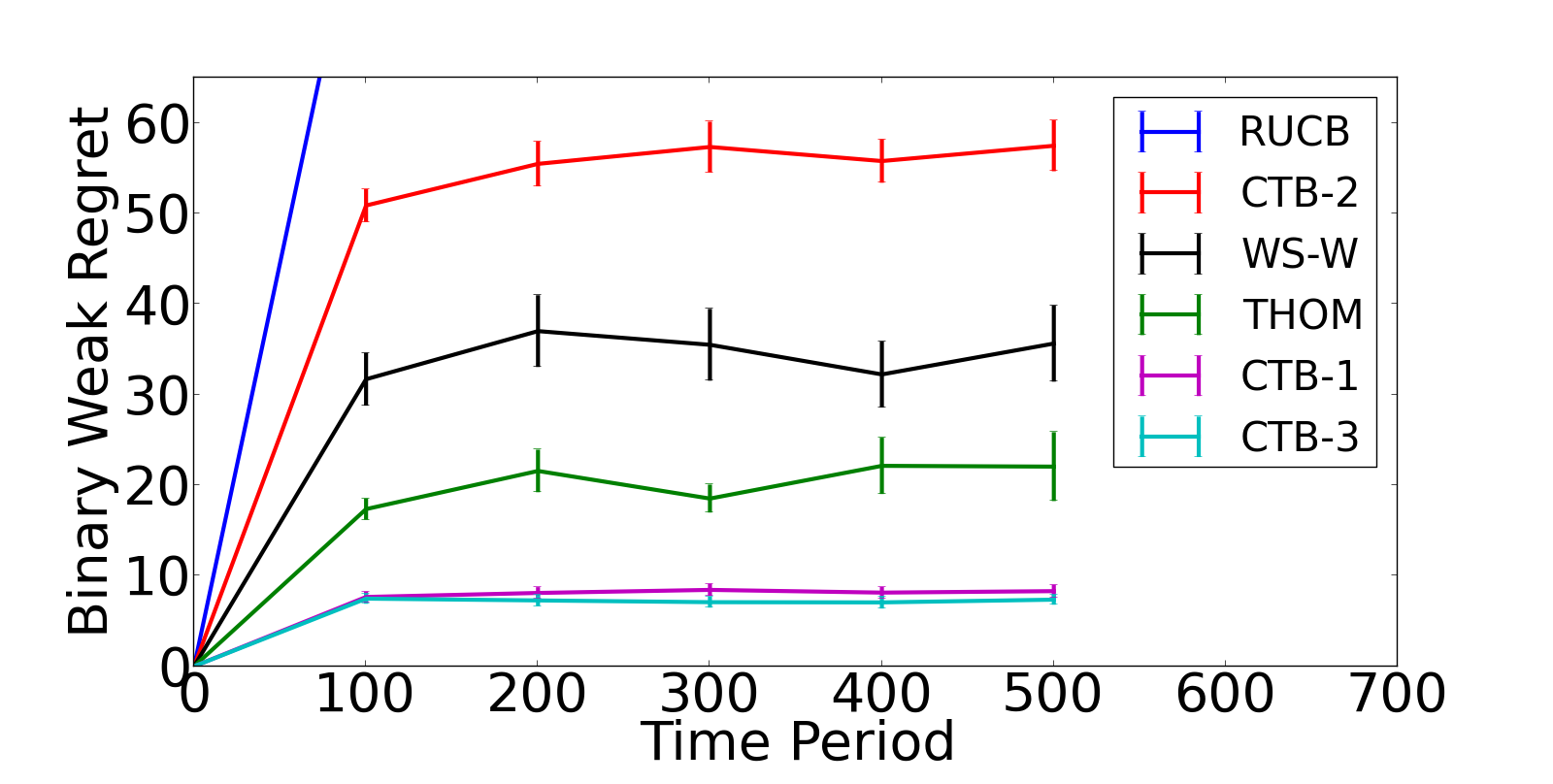}
        \caption{Binary Regret and Constant $p_{i,j}$}   
        \label{fig:binary}
        \end{subfigure}%
    ~ 
    \begin{subfigure}[t]{0.5\textwidth}
        \centering
        \includegraphics[width=1\textwidth]{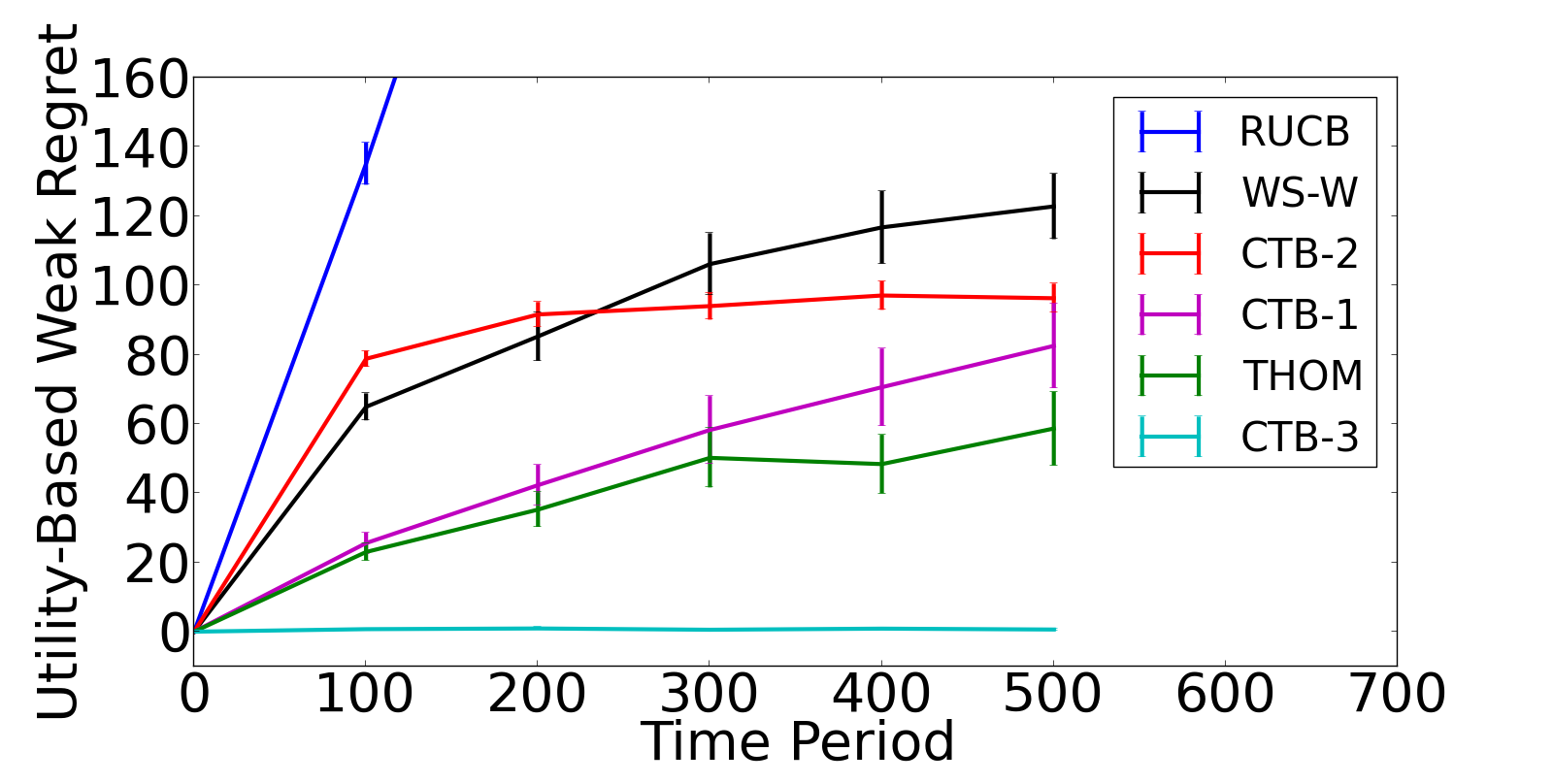}
        \caption{Bradley-Terry Regret and $p_{i,j}$}
        \label{fig:utility}
    \end{subfigure}
    \caption{
Performance comparison of the three \CTB\ variants from section~\ref{Methods} against benchmarks WS-W, RUCB and Thompson Sampling (THOM) using simulated datasets. \CTBthree\ and Thompson sampling use prior information, and in this group \CTBthree\ performs best. Among the four algorithms that do not use prior information, \CTBone\ performs best. \CTBtwo\ under-performs WS-W in the binary regret setting and for $t=100,200$ in the Bradley-Terry setting, and outperforms WS-W when $t=300,400,500$ in the Bradley-Terry setting.}
\label{fig:result2}
\end{figure*}

Since RUCB performs poorly in both experiments compared with other algorithms, we set the $y$-axis to emphasize the relative performance of the other algorithms. We include a plot over a wider $y$-axis showing RUCB's performance in the supplement.

\subsection{Binary Regret and Constant $p_{i,j}$}
\label{sec:binary}
In this experimental setting, we set $p_{i,j}=0.8$ for all $i<j$. We have $N=20$ arms uniformly generated from the $2$-dimensional unit circle. The preference vector $\theta$ is generated uniformly at random from the $2$-dimensional unit circle. We set regret to $1$ if both of the pulled arms are not optimal, i.e. $u(\theta,A_1)=1$ and $u(\theta,A_i)=0$ for $i\neq 1$.  To satisfy our previous assumption that $u(\theta,A_i)$ be distinct across $i$, we may equivalently set $u(\theta,A_i)=i\cdot\epsilon$, and take $\epsilon$ small.

Figure~\ref{fig:binary} shows that \CTBone\ and \CTBthree\ perform comparably and both outperform WS-W and Thompson Sampling. \CTBtwo\ does not perform as well as WS-W and Thompson Sampling.
Both Thompson sampling and \CTBthree\ have access to the correct prior and use the true value of $p$ to perform updating.

% In both experiments we select $\theta$ at random from a uniform Bayesian prior density, and the Bayesian algorithms (CTB-3 and Thompson sampling) have access to this prior.  The other algorithms to do not.  Access to $p_{i,j}$ for updating is discussed below.  

\subsection{Bradley-Terry Regret and $p_{i,j}$}
\label{sec:utility}

In this experimental setting, we set utility using the Bradley-Terry model described in section~\ref{probForm}. As in the first experimental setting, we have $N=20$ arms on the $2$-dimensional unit circle. Among these arms, $19$ are uniformly generated from $\{x<0,y<0,x^2+y^2=1\}$ and $1$ arm is uniformly generated from $\{x>0, y>0, x^2+y^2=1\}$. The user's preference $\theta$ is also uniformly generated from $\{x>0, y>0, x^2+y^2=1\}$, but the Bayesian algorithms (\CTBthree and Thompson sampling) use another less information prior: that $\theta$ is uniform on the unit circle.  Thompson sampling performs its update using the true $p_{i,j}$, while \CTBthree\ uses a rough approximation of $q=0.6$ to set $m_i(0)$ to model the fact that we would not know $p$ or $p_{i,j}$ in practice.

% When applying \CTBthree, we need to set $m_i$ so that it represents the prior information. In this experiment, similar to section~\ref{sec:Bayes}, we set $m_i=\frac{\log(p_{0}(C_i)}{\log(\frac{q}{1-q})}$. In this experiment (as well as in many real applications), we do not know $q$ and the probability of $p_{i,j}$ may not be a constant. However, based on the proof of Theorem~\ref{thm:1}, our regret bound holds true as long as $m_1>-\infty$. In this specific experiment, we choose $q=0.6$. This choice of $m_i$ provides a approximation of the prior information we have.

Figure~\ref{fig:utility} shows that both \CTBthree\ and Thompson Sampling takes advantage of the prior information and the dependence among arms. \CTBthree\ uses this information more efficiently and significantly outperforms Thompson Sampling. Among the four algorithms (\CTBone, \CTBtwo, RUCB and WS) that do not use prior information, \CTBone\ performs best. Though \CTBtwo\ does not perform as well as WS at $t=100, 200$, it outperforms WS when $t=300,400,500$. 

\section{Conclusion}
In this paper, we consider dueling bandits for weak regret, with application to recommender systems and online content recommendation. We formulate a new setting which differs from the traditional dueling bandits in which arms are dependent. We propose an algorithm \CTB, and show it has constant expected cumulative regret and strong empirical performance.
\newpage

\appendix
\section*{Appendix A.}

\subsection*{Proof of Lemma~\ref{basic}}

First we prove another lemma.
\begin{lemma}
Suppose $Z(k)$ is a random walk starting with $Z(0)=0$, $Z(k+1)=Z(k)+1$ with probability $p>0.5$ and $Z(k+1)=Z(k)-1$ with probability $1-p$. Then for $S\in\mathbb{N}$ we have
\begin{align}
E\left[\sum_{t=0}^{\infty}\mathbbm{1}\{Z(t)\leq S\}\right] =\frac{p+S(2p-1)}{(2p-1)^{2}}.
\end{align}
\end{lemma}

\begin{proof}
Denote $A=\mathbb{E}[t:\min_{t>1}Z(t)=0|Z(1)=-1]$ and $B=P(\exists t, Z(t)=0|Z(1)=1)$, then we know
\begin{align}
\mathbb{E}\left[\sum_{t=0}^{\infty}\mathbbm{1}\{Z(t)\leq 0\}\right]=1+(1-p)\left(A+\mathbb{E}\left[\sum_{t=0}^{\infty}\mathbbm{1}\{Z(t)\leq 0\}\right]\right)+pB\mathbb{E}\left[\sum_{t=0}^{\infty}\mathbbm{1}\{Z(t)\leq 0\}\right]. \nonumber 
\end{align}
Now we need to calculate the expression for A and B respectively.

Based on the definition of A, we can rewrite A as $\mathbb{E}[t:\min_{t>1}Z(t)=1|Z(t)=0]$. It is easy to show that $Y(t):=Z(t)-(2p-1)t$ is a martingale. Here we define a stopping time $\tau$ as $\min\{t>1:Z(1)=1\}$. Then we know $Y(t)$ stops at $\tau$ is a martingale and thus $\mathbb{E}[Y(\tau)]=\mathbb{E}[Z(\tau)]-(2p-1)E[\tau]=0$. Thus $A=\frac{1}{2p-1}$.

For B, based on the first step analysis, we know
\begin{align}
B = (1-p) + p\times B^{2}. \nonumber
\end{align}
Solving this equation, we get $B=\frac{1-p}{p}$. 

Plus in A and B's expression, we have
\begin{align}
\mathbb{E}\left[\sum_{t=0}^{\infty}\mathbbm{1}\{Z(t)\leq 0\}\right]=\frac{p}{(2p-1)^{2}}. \nonumber
\end{align}
Now we compute $\mathbb{E}\left[\sum_{t=0}^{\infty}\mathbbm{1}\{Z(t)\leq 1\}\right]$. Based on the same reasoning, we know
\begin{align}
\mathbb{E}\left[\sum_{t=0}^{\infty}\mathbbm{1}\{Z(t)\leq 1\}\right]=1+(1-p)\left(A+\mathbb{E}\left[\sum_{t=0}^{\infty}\mathbbm{1}\{Z(t)\leq 1\}\right]\right)+p\times \mathbb{E}\left[\sum_{t=0}^{\infty}\mathbbm{1}\{Z(t)\leq 0\}\right]. \nonumber
\end{align}
Solving it, we get $\mathbb{E}\left[\sum_{t=0}^{\infty}\mathbbm{1}\{Z(t)\leq 1\}\right]=\frac{p+(2p-1)}{(2p-1)^2}$. For general $S$, we have
\begin{align}
\mathbb{E}\left[\sum_{t=0}^{\infty}\mathbbm{1}\{Z(t)\leq S\}\right]=1+(1-p)\left(A+\mathbb{E}\left[\sum_{t=0}^{\infty}\mathbbm{1}\{Z(t)\leq S\}\right]\right)+p\times \mathbb{E}\left[\sum_{t=0}^{\infty}\mathbbm{1}\{Z(t)\leq S-1\}\right], \nonumber
\end{align}
by induction, we know our Lemma is true.
\end{proof}

\begin{comment}
\begin{lemma}
Suppose $Z(k)$ is a random walk starts with $Z(0)=0$.  $Z(k+1)=Z(k)+1$ with probability $p>0.5$ and $Z(k+1)=Z(k)-1$ with probability 1-p, then for $\Delta\geq 0$, we have
\begin{equation}
\mathbb{E}\left[\sum_{t=0}^{\infty}\mathbbm{1}\{Z(t)\leq \Delta\}\right] =\frac{p+\Delta(2p-1)}{(2p-1)^{2}}.
\end{equation}
\end{lemma}

\begin{proof}
Denote $C=\mathbb{E}[\min_{t}:Z(t)=\Delta]$ and $D=\mathbb{E}[\sum_{t=0}^{\infty}\mathbbm{1}\{Z(t)\leq \Delta\}|Z(0)=\Delta]$. Since $\Delta\geq 0$, we have
\begin{equation}
    E[\sum_{t=0}^{\infty}\mathbbm{1}\{Z(t)\leq \Delta\}]=C+D, 
\end{equation}
and we know $D=E[\sum_{t=0}^{\infty}\mathbbm{1}\{Z(t)\leq 0\}|Z(0)=0]=\frac{p}{(2p-1)^{2}}$.

Because $Y(t):=Z(t)-(2p-1)t$ is a martingale and $\tau=\min{t:Z(t)=\Delta\}$ is a stopping time, we know $Y(\tau)$ is a martingale. Thus, $E[Y(\tau)]=Z(\tau)-(2p-1)E[\tau]=E[Y(0)]=0$ and $E[\tau]=\frac{\Delta}{2p-1}$.

Thus, we know $E[\sum_{t=0}^{\infty}\mathbbm{1}\{Z(t)\leq \Delta\}]=\frac{p+\Delta(2p-1)}{(2p-1)^{2}}$.
\end{proof}
\end{comment}

Now we return to the proof of Lemma 1.
\begin{proof}
Suppose W(t) is a random walk and $W(t+1)=W(t)+1$ with probability p and $W(t+1)=W(t)-1$ with probability 1-p. Based on the previous Lemma, we just need to show
\begin{align}
\mathbb{E}\left[\sum_{t=0}^{\infty}\mathbbm{1}\{Z(t)\leq S\}\right]\leq \mathbb{E}\left[\sum_{t=0}^{\infty}\mathbbm{1}\{W(t)\leq S\}\right].
\label{bound}
\end{align}
Because $E[\sum_{t=0}^{\infty}\mathbbm{1}\{W(t)\leq S\}]=\sum_{t=0}^{\infty}P(W(t)\leq S)$ and 
\begin{align}
P(W(t)\leq S)=\sum_{2m\geq t-S}{t \choose m}p^{t-m}(1-p)^{m}\geq P(Z(t)\leq S), \nonumber
\end{align}
we know Equation~\ref{bound} holds true.
\end{proof}

\subsection*{Proof of Lemma~\ref{lemma:bayes}}

\begin{proof}
We first prove it for $Y_t=0$ and $x\in H_{i,j}$. This is because
\begin{align}
p_{t+1}(x) &= p_{t+1}(\theta\in x) \nonumber \\
&=P(\theta\in x| Y_t=0, p_t(\cdot)) \nonumber \\
&=\frac{P(\theta\in x, Y_t=0, p_t(\cdot))}{P(Y_t=0, p_t(\cdot))} \nonumber \\
&=\frac{P(\theta\in x, Y_t=0, p_t(\cdot))}{P(Y_t=0, p_t(\cdot)|\theta \in H_{i,j})P(\theta \in H_{i,j})+P(Y_t=0, p_t(\cdot)|\theta \notin H_{i,j})P(\theta \notin H_{i,j})} \nonumber \\
&=\frac{p_{t}(x)q}{p_t(H_{i,j})q+(1-p_t(H_{i,j}))(1-q)}. \nonumber 
\end{align}

The other three cases follow the same reasoning and we omit the proof.
\end{proof}

\subsection*{Proof of Lemma~\ref{posterior}}

\begin{proof}
We prove this lemma using induction. This is obviously true when t=0. Suppose this is true at time t-1. Without loss of generality, we write
\begin{align}
p_{t-1}(C_k)=\frac{p_{0}(C_{i})q^{m_{i}(t-1)-m_{i}(0)}(1-q)^{t-1-m_{i}(t-1)+m_{i}(0)}}{M(t-1)}, \nonumber
\end{align}
where $M(t-1)$ is a scaling constant. At time t, suppose we choose $A_{i}$ and $A_{j}$ for comparison and $A_{i}$ wins the duel. Denote $M(t)=M(t-1) * [p_{t-1}(H_{i}{j}) * q + (1-p_{t-1}(H_{i,j}))(1-q)]$, then if $C_{k}\in H_{i,j}$:
\begin{align}
     p_{t}(C_{k}) &= \frac{p_{t-1}(C_{k})q}{p_{t-1}(H_{i,j})q+(1-p_{t-1}(H_{i,j}))(1-q)} \nonumber \\
     &=\frac{p_{0}(C_{k})q^{m_{k}(t-1)-m_{k}(0)}(1-q)^{t-1-m_{k}(t-1)+m_{k}(0)}q}{M(t-1)[p_{t-1}(H_{i,j})q+(1-p_{t-1}(H_{i,j}))(1-q)]} \nonumber \\
     &=\frac{p_{0}(C_{k})q^{m_{k}(t)-m_{k}(0)}(1-q)^{t-m_{k}(t)+m_{k}(0)}}{M(t)}, \nonumber 
\end{align}
where the last line is based on the definition of $m_{k}(t)$ and $M(t)$. Similarly, if $C_{k}\notin H_{i}{j}$, then
\begin{align}
    p_{t}(C_{k}) &= \frac{p_{t-1}(C_{k})(1-q)}{p_{t-1}(H_{i,j})q+(1-p_{t-1}(H_{i,j}))(1-q)}   & \nonumber \\
    &=\frac{p_{0}(C_{k})q^{m_{k}(t-1)-m_{k}(0)}(1-q)^{t-1-m_{k}(t-1)+m_{k}(0)}(1-q)}{M(t-1)[p_{t-1}(H_{i,j})q+(1-p_{t-1}(H_{i,j}))(1-q)]} \nonumber \\
     &=\frac{p_{0}(C_{k})q^{m_{k}(t)-m_{k}(0)}(1-q)^{t-m_{k}(t)+m_{k}(0)}}{M(t)}. \nonumber
\end{align}

\end{proof}

\subsection*{Full Plot of Section~\ref{sec:exp}}
We include a plot which contains full information for RUCB. See Figure~\ref{fig:result3} for details.

\begin{figure*}[!h]
    \centering
    \begin{subfigure}[t]{0.5\textwidth}
        \centering
        \includegraphics[width=1\textwidth]{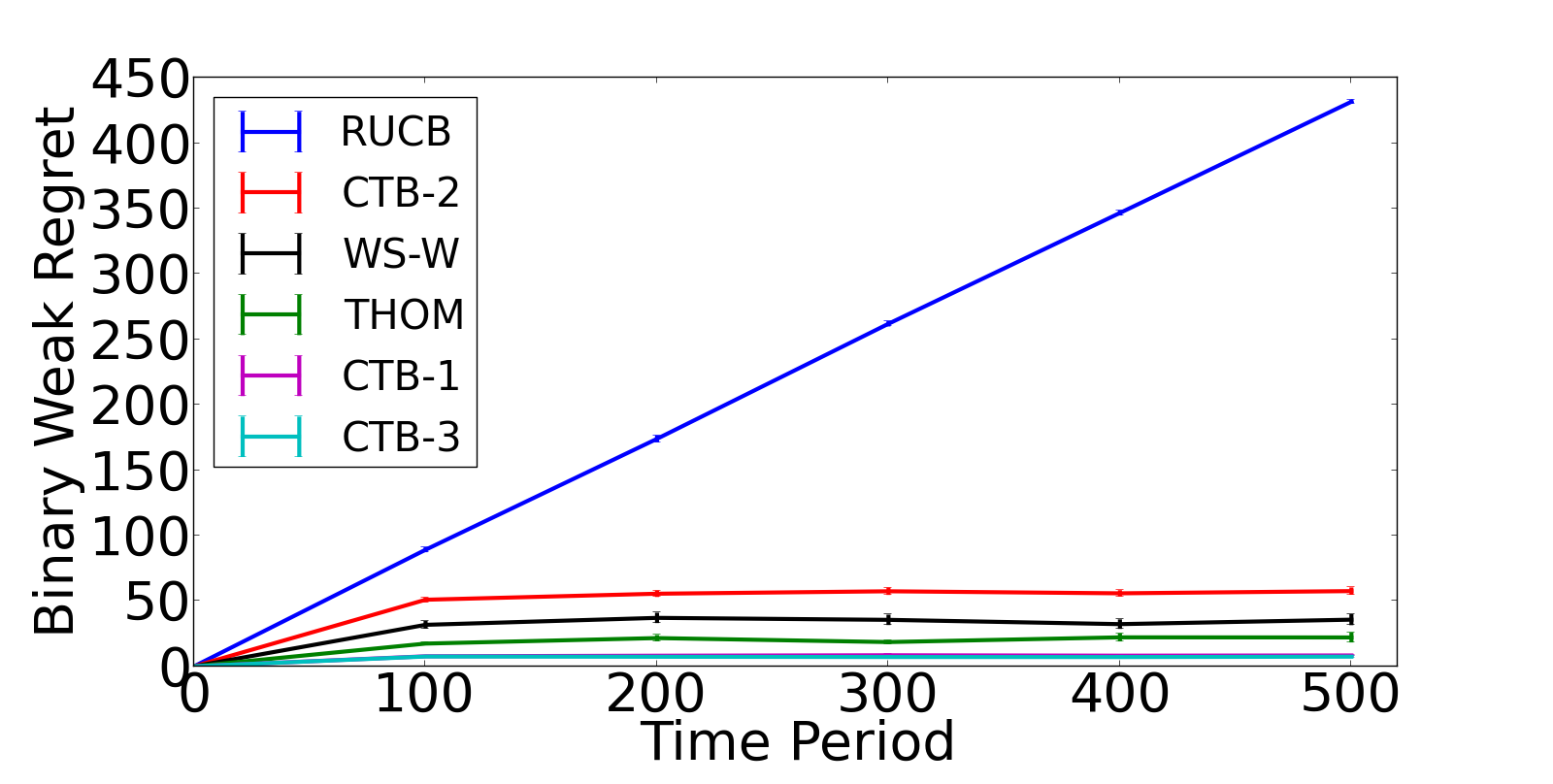}
        \caption{Binary Regret and Constant $p_{i,j}$}   
        \end{subfigure}%
    ~ 
    \begin{subfigure}[t]{0.5\textwidth}
        \centering
        \includegraphics[width=1\textwidth]{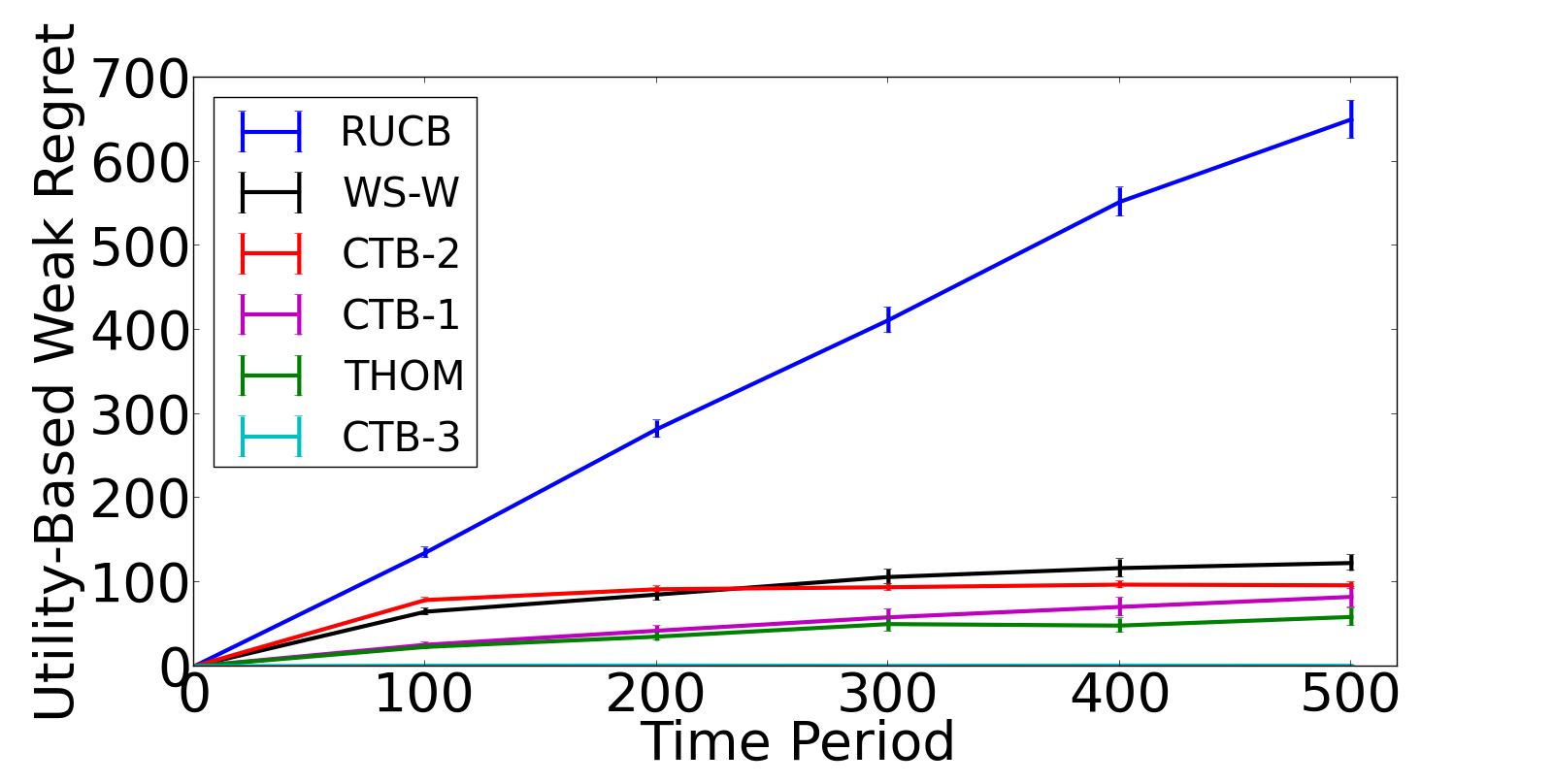}
        \caption{Bradley-Terry Regret and $p_{i,j}$}
    \end{subfigure}
    \caption{Performance comparison of \CTBone, \CTBtwo, \CTBthree, WS, RUCB and Thompson Sampling in the same experimental settings as in section~\ref{sec:exp}, but with plots containing full information for RUCB.}
\label{fig:result3}
\end{figure*}

\vskip 0.2in
\bibliography{sample}

\end{document}